\documentclass[letterpaper, 10 pt, journal, twoside]{ieeetran}

\usepackage[letterpaper, left=0.67in, right=0.67in, top=0.79in, bottom=1.0in]{geometry}
\usepackage{amsmath,amsfonts,amssymb}
\usepackage{algorithmic}
\usepackage{algorithm}
\usepackage{array}
\usepackage[caption=false,font=normalsize,labelfont=sf,textfont=sf]{subfig}
\usepackage{textcomp}
\usepackage{stfloats}
\usepackage{url}
\usepackage{verbatim}
\usepackage{graphicx}
\usepackage{amsthm}
\usepackage{cite}
\usepackage{booktabs}
\usepackage{multirow}

\usepackage{gensymb}
\usepackage{booktabs}

\newtheorem{lemma}{Lemma}

\begin{document}

\title{Certifiable Alignment of GNSS and Local Frames \\ via Lagrangian Duality}

\author{Baoshan Song, Matthew Giamou, Penggao Yan, Chunxi Xia, Li-Ta Hsu,~\IEEEmembership{Senior member,~IEEE,}
\thanks{This paper was produced by the IEEE Publication Technology Group. They are in Piscataway, NJ.}
\thanks{Manuscript received April 19, 2021; revised August 16, 2021.}}

\author{Baoshan Song$^{1}$, Matthew Giamou$^{2}$, Penggao Yan$^{1}$, Chunxi Xia$^{3}$, and Li-Ta Hsu$^{1}$%
\thanks{Manuscript received: December, 23, 2025; Revised March, 6, 2026; Accepted April, 27, 2026.}
\thanks{This paper was recommended for publication by Editor Editor Ayoung Kim upon evaluation of the Associate Editor and Reviewers’ comments.
This work was supported in part by the Research Grants Council of the Hong Kong Special Administrative Region, China, under the Collaborative Research Fund, for the project ``Heterogeneity-aware Collaborative Edge AI Acceleration (AAE)'' (Funding Body Ref. No. C5032-23G; Project ID: P0051348).} 
\thanks{$^{1}$Baoshan Song, Penggao Yan and Li-Ta Hsu are with the Department of Aeronautical and Aviation
Engineering, The Hong Kong Polytechnic University, China.
        {\tt\footnotesize baoshan.song@connect.polyu.hk, \{penggao.yan,lt.hsu\}@polyu.edu.hk}}%
\thanks{$^{2} $Matthew Giamou is with Department of Computing and Software, McMaster University, Canada.
        {\tt\footnotesize giamoum@mcmaster.ca}}%
\thanks{$^{3} $Chunxi Xia is with the School of Geodesy
and Geomatics, Wuhan University, China.
        {\tt\footnotesize xiachunxi@whu.edu.cn}}%
\thanks{Digital Object Identifier (DOI): see top of this page.}
}


\markboth{IEEE Robotics and Automation Letters. Preprint Version. Accepted APRIL, 2026}
{Song \MakeLowercase{\textit{et al.}}: Certifiable Alignment of GNSS and Local Frames via Lagrangian Duality} 


\maketitle

\begin{abstract}
Estimating the absolute orientation of a local system relative to a global navigation satellite system (GNSS) reference often suffers from local minima and high dependency on satellite availability.  Existing methods for this alignment task rely on abundant satellites unavailable in GNSS-degraded environments, or use local optimization methods which cannot guarantee the optimality of a solution. This work proposed a certifiable method, meaning it can numerically verify the optimality of the result, filling a gap where existing local optimizers fail. 
 We first formulate the original Doppler-based frame alignment problem as a nonconvex quadratically constrained quadratic program (QCQP) problem and relax the QCQP problem to a concave Lagrangian dual problem that provides a lower cost bound for the original problem. Then we perform relaxation tightness and observability analysis to derive criteria for certifiable optimality of the solution. Finally, simulation and real world experiments are conducted to evaluate the proposed method. The experiments show that our method provides certifiably optimal solutions even with only 2 satellites with Doppler measurements and 2D vehicle motion, while the traditional velocity-based VOBA method and the advanced GVINS alignment technique may fail or converge to local optima without notice. To support the development of GNSS-based navigation techniques in robotics, all the code and data are open-sourced at \url{https://github.com/Baoshan-Song/Certifiable-Doppler-alignment}. 
\end{abstract}


\section{INTRODUCTION}
\IEEEPARstart{L}{ocalization} is an essential capability for any autonomous mobile system\cite{groves_principles_2013}. Nevertheless, many dead-reckoning (DR) based systems, such as visual-inertial navigation systems (VINS)\cite{cao_gvins_2022} and lidar-inertial odometry (LIO)\cite{li_p3-lins_2023}, operate by integrating incremental local measurements.
From the perspective of consistency analysis, these DR-based systems inherently possess four unobservable degrees of freedom (DoF) in the absence of external absolute references: the 3D global position and the rotation around the gravity vector (i.e., the yaw or heading angle)\cite{zhou_mast3r-fusion_2025}. Without periodic absolute corrections, the errors in these unobservable dimensions accumulate indefinitely, leading to drift and significant system inconsistency.
To mitigate this, incorporating Global Navigation Satellite System (GNSS) measurements is a well-established strategy to anchor these unobservable manifolds to a globally consistent frame. However, to effectively fuse GNSS with local DR systems, the autonomous platform must first solve an initial frame alignment problem.

Since GNSS can easily provide globally consistent position \cite{wen_robust_2021}, the major problem is to estimate the rotation alignment parameters. Thus we refer to the initial rotation alignment between GNSS and local frames as \emph{initial alignment}. In this work, we focus on this ``cold-start'' initialization problem, where no prior orientation is available. Since rotations in three dimensions are described by the nonconvex special orthogonal group ($\mathrm{SO(3)}$), the principal scientific challenge of the initial alignment problem is avoiding local optima which lead to poor alignment performance. To mitigate the effect of nonconvexity, there are plenty of existing works including traditional methods using motion synchronization and advanced methods using raw GNSS measuring models in the following sections. 

\begin{figure}
    \centering
    \includegraphics[width=1\linewidth]{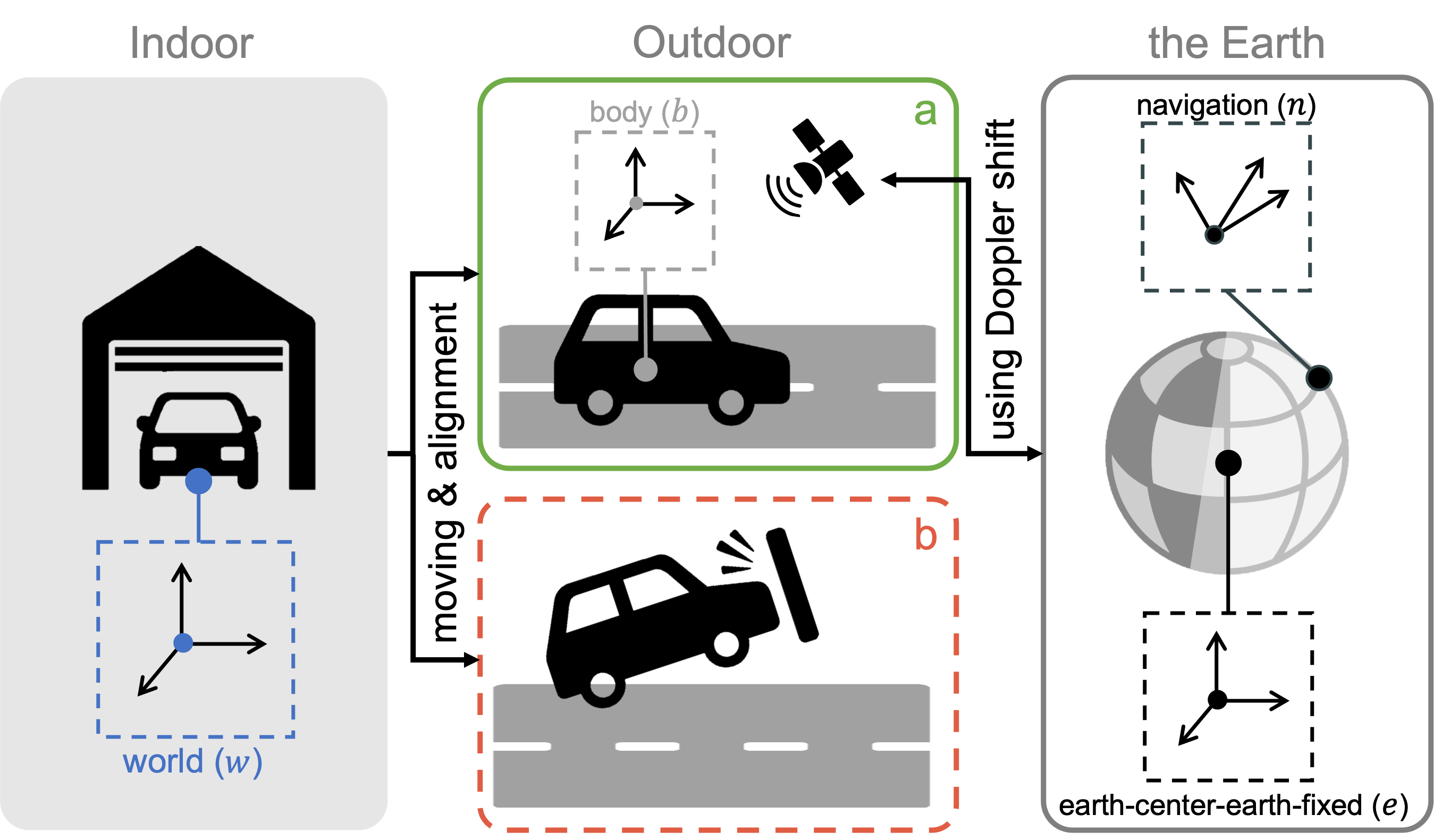}
    \caption{
    Illustration of the coordinate alignment between the local $w$-frame and global $e/n$-frames. The key insight is to align the local velocity with GNSS Doppler measurements to resolve global orientation. Definitions: $w$ is the local world frame fixed at the start point; $b$ is the moving body frame; $n$ is the navigation frame; and $e$ is the earth-center-earth-fixed (ECEF) frame. (a) Successful alignment: Leveraging the certifiable global optimum ensures accurate positioning for safe autonomous driving. (b) Alignment failure: Misalignment leading to drift or hazardous lane departures.}
    \label{fig:demo}
\end{figure}

\subsection{Traditional alignment methods: closed-form solution}
GNSS-aided local frame alignment has been traditionally approached as a motion synchronization problem, i.e., a Wahba problem\cite{wahba_1965}, using motion states such as position\cite{chen_rapid_2023}, velocity\cite{zhang_vboba_2020} and acceleration\cite{wu_optimization_2011} estimated from a GNSS receiver and other sensors. For instance, optimization-based alignment (OBA) is a term used for systems performing GNSS/Inertial Measurement Units (IMU) integrated alignment\cite{wu_optimization_2011}. Since there are many mature solutions for motion synchronization, these initial alignment methods can be addressed by solving registration problems (e.g., Wahba's problem) using computational methods such as Singular Value Decomposition (SVD)\cite{markley_attitude_1993}. For example, the position-based synchronization can be solved by Arun's method to achieve both translation and rotation alignment\cite{arun_least-squares_1987}. Formulations using velocity measurements can also be solved by existing methods, but only provide rotation information\cite{umeyama_least-squares_2002}. 
In short, motion synchronization-based methods can be solved to a closed-form optimal solution. However, these methods require abundant ($\ge4$) GNSS satellites to provide motion states as input, which is challenging in GNSS-degraded environments \cite{groves_principles_2013}. While simple and optimal, the effectiveness of these methods diminishes under sparse or degraded GNSS conditions, leading to the first research gap.

\subsection{Advanced alignment methods: raw GNSS measurement based methods}
To overcome the reliance on rich satellite geometry, recent methods have shifted toward frame alignment using raw GNSS measurements. These include carrier-phase and pseudorange-based positioning, as well as signal shifting measurements such as time-differenced carrier phase (TDCP)\cite{zhang_carrier-phase_2022} and Doppler-based\cite{wei_carrier_2021} velocity estimation. These methods enable alignment even in low-satellite-visibility environments and shorter time spans, by leveraging the geometric consistency between local motion estimates and GNSS observations. However, to the best of our knowledge, there are not closed-form solutions for the alignment problem using raw GNSS measurements, and existing raw GNSS-based methods using local search may converge to local optima when given poor initial guesses. Specifically, these methods offer no optimality guarantees, meaning their convergence is highly sensitive to the quality of the initial guess. Without a certifiable framework, it is impossible to distinguish between a globally optimal alignment and a suboptimal local attractor based on the measurement residuals alone. This leads to the second research gap. 

Except from general methods, the extensive body of research on inertial navigation systems (INS) has led to a large number of frame alignment methods designed particularly for inertial sensors. Since a low-cost IMU can provide roll and pitch angles during quasi-static motion by sensing the local direction of the Earth's gravitational field, these alignment methods mainly focus on heading alignment\cite{cao_gvins_2022, li_p3-lins_2023}. 
With determined horizontal angles, the rotation part of the feasible region shrinks from $\mathrm{SO(3)}$ to $\mathrm{SO(2)}$. This reduction significantly eases the difficulty of local searching, as $\mathrm{SO(2)}$ is a minimal degree variety where any non-negative quadratic objective is a sum-of-squares\cite{brynte_tightness_2022}. This geometric property implies a more tractable optimization landscape that allows local search methods to converge to the global optimum with much higher probability.


\subsection{Our porposed approach: convex relaxation with Lagrangian duality}
        In this work, we aim to bridge these research gaps by developing a certifiable alignment method using raw GNSS measurements that remains robust under partial observability and provides optimality guarantees. To achieve this objective, it is contributive to introduce the growth of convex relaxation. 

If a nonconvex problem can be relaxed into a convex problem, it is possible to efficiently find a global optimum. The theory of convex optimization has been studied for decades but limited by computing resources. In recent years, with the rapid evolution of computing devices, advanced optimization methods are leveraged for large scale optimization problems and relaxed problems can be solve efficiently\cite{boyd_convex_2004}. As a result, in the last decade, convex relaxation has been applied to more robotics problems, such as pose graph optimization (PGO)\cite{rosen_se-sync_2019}, indoor positioning\cite{deng_doppler_2018}, geometric registration\cite{yang_teaser_2020}, hand-eye calibration\cite{giamou_certifiably_2019}, rotation synchronization\cite{liu_resync_2023}, wireless localization\cite{yan_robust_2022}, camera arrangement\cite{kaveti_oasis_2024}, orbit determination\cite{song_certifiably_2025b}, and opportunistic positioning\cite{song_certifiably_2025a}. 

This work proposes a certifiable GNSS/local initial frame alignment method using raw Doppler shift measurements via semidefinite programming (SDP). First, we use the GNSS Doppler shift model for alignment and construct a nonconvex QCQP problem. This nonconvex problem is solved using a convex Lagrangian dual relaxation, which we prove is guaranteed to provide a globally optimal solution under explicit observability conditions. 
To the best of our knowledge, this is the first certifiably optimal GNSS/local frame alignment method using raw GNSS Doppler shift measurements. 
The contributions are summarized below:
\begin{itemize}
    \item \textbf{Certifiable GNSS-aided alignment method}: a novel estimation method for Doppler-aided GNSS/local frame alignment using convex relaxation. This method provides a solution with an optimality guarantee. 
    \item \textbf{Tightness and observability analysis }: relaxation tightness analysis and observability analysis are performed to cope with noisy and degraded measurements in GNSS-degraded environments, such as urban canyons. We find the necessary conditions to guarantee observability and sufficient-and-necessary conditions to guarantee optimality. 
    \item \textbf{Experimental evaluation}: both simulation and real world tests are performed to evaluate the proposed method. The results show that our method outperforms the traditional velocity synchronization based method VOBA \cite{zhang_vboba_2020} and the advanced alignment method used in GVINS \cite{cao_gvins_2022}. We also open-source our code and data for free use by the robotics and autonomous navigation communities.
\end{itemize}

The remainder of this paper is organized according to the pipeline shown in Fig. \ref{fig:pipeline}. Section \ref{sec:problem_formu} introduces the original Doppler-based GNSS/local frame alignment model and construct an equivalent QCQP problem. Section \ref{sec:alignment} employs Lagrangian duality to solve the QCQP problem and derive the necessary conditions for optimality and observability. Section \ref{sec:experiment} evaluates the performance of the proposed method via simulation and real data. Finally, Section \ref{sec:conclusion} discusses our conclusions and future work.

\begin{figure}
    \centering
    \includegraphics[width=1\linewidth]{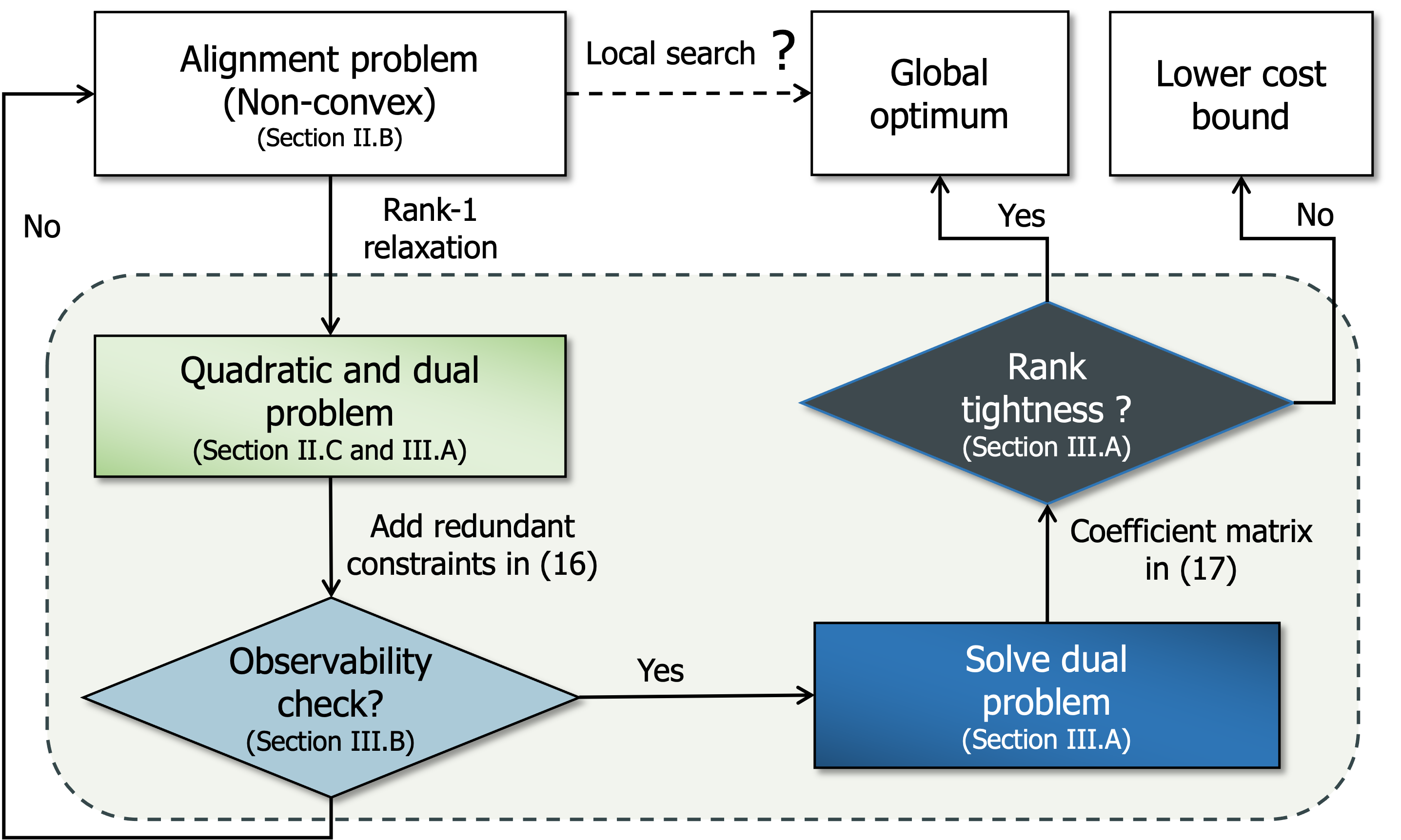}
    \caption{Pipeline of certifiable alignment}
    \label{fig:pipeline}
\end{figure}
\section{Problem Formulation}
\label{sec:problem_formu}
\subsection{Notation}
Vectors and matrices are denoted by lowercase and uppercase letters with boldface (e.g. $\mathbf{a}$ and $\mathbf{A}$). A diagonal matrix with diagonal elements $a_1,...,a_n$ is denoted by $\text{diag}\{a_1,...,a_n\}$. The $i$-th to $j$-th components of $\mathbf{a}$ is denoted by  $\mathbf{a}_{(i:j)}$, and the submatrix of $\mathbf{A}$ with corners $(i,i),(i,j),(j,i),(j,j)$ is denoted by $\mathbf{A}_{(i:j)}$. A $n\times n$ identity matrix is denoted by $\mathbf{I}_n$. The rank and trace of square matrix $\mathbf{A}$ are denoted by $\text{rank}(\mathbf{A})$ and $\text{tr}(\mathbf{A})$ respectively. An $n$-dimension real space is denoted by  $\mathbb{R}^n$ and an $n\times n$ symmetric matrix space is denoted by $\mathbb{S}^n$. $\mathbf{A}\succeq 0$ indicates that $\mathbf{A}$ is positive semidefinite and $\mathbf{A}\succ 0$ indicates that $\mathbf{A}$ is positive definite. $\left \| \mathbf{A}\right \|$ and $\left \| \mathbf{a}\right \|$ denote the 2-norm of a matrix $\mathbf{A}$ and a vector $\mathbf{a}$, respectively. The Frobenius inner product of $\mathbf{A}$ and $\mathbf{B}$ is $\mathbf{A}\bullet\mathbf{B}$. $\otimes$ and $\odot$ denote kronecker and hadamard products, respectively. 


\subsection{Original problem}
Our problem is close to the coarse GNSS initial alignment method in GVINS \cite{cao_gvins_2022}. In GVINS, the yaw offset between the ENU frame and the local world frame is estimated using Doppler measurements. To generalize this method, we extend the Doppler-based rotation alignment between ECEF and local world frame to ${\mathrm{SO(3)}}$ more than yaw offset. According to\cite{guo_instantaneous_2023}, a typical raw GNSS Doppler measurement model can be written as:
\begin{equation}
\begin{array}{cc}
         \lambda \cdot D= (\mathbf{p}_r^e-\mathbf{p}_s^e)^T(\mathbf{v}_r^e-\mathbf{v}_s^e)/\left\|\mathbf{p}_r^e-\mathbf{p}_s^e\right\|_2    + c\cdot d\dot{t}_r +\epsilon, 
\end{array}
\label{equ:doppler_true}
\end{equation}
where $\lambda$ is the wavelength of the GNSS carrier signal ($\mathrm{m}$); $D$ is a Doppler shift measurement ($\mathrm{Hz}$); $\mathbf{p}_r^s$ and $\mathbf{p}_s^e$ are the position of the receiver and the satellite ($\mathrm{m}$);  $\mathbf{v}_r^e$ and $\mathbf{v}_s^e$ are the velocity of the receiver and the satellite ($\mathrm{m/s}$); $c$ is the speed of light ($\mathrm{m/s}$); $d\dot{t_r}$ is the clock drift of the receiver; and $\epsilon$ is additive Gaussian white noise ($\mathrm{m/s}$).  Geometrically, a Doppler shift measurement is the projection of relative velocity on the line of sight vector between a satellite and a receiver \cite{song_certifiably_2025a}. Since a Doppler measurement is generally less noisy than a pseudorange measurement, it can be used to determine the rotation between the ECEF frame and the local world frame $\mathbf{R}^e_w$ using local receiver velocity $\mathbf{v}^w_r$ in the $w$-frame: 
\begin{gather}
    \mathbf{v}_r^e = \mathbf{R}^e_w \mathbf{v}_r^w.
\end{gather}

Also, we simplify the notation of the Doppler shift model with
\begin{gather}
        \mathbf{n} \triangleq (\mathbf{p}_r^e-\mathbf{p}_s^e) / \|\mathbf{p}_r^e-\mathbf{p}_s^e\|_2, \\
    \bar{D} \triangleq \lambda \cdot D - \mathbf{n}^T\mathbf{v}^e_s, \quad \bar{t} \triangleq c \cdot d\dot{t}_r.
\end{gather}
Given a bundle of input Doppler shift measurements, satellite positions and velocities from ephemeris, and local velocities at the same time, we can get the Doppler-based GNSS/local frame alignment problem:
\begin{equation}
\begin{aligned} \label{equ:original}
    &\min_{\mathbf{R}^e_w\in \mathrm{SO(3)},\bar{t}\in\mathbb{R}}\sum_{i=1}^K||z_i||^2_\epsilon
    \\     &z_i=(\mathbf{R}^e_w\mathbf{v}^w_{r,i}-\mathbf{v}^e_{s,i})^T\mathbf{n}_i+\bar{t}-\bar{D}_{i}.
\end{aligned}
\end{equation}
Problem (\ref{equ:original}) is nonconvex since the $\mathrm{SO(3)}$ is a nonconvex set. 
In general, existing local search methods can efficiently find solutions to problems of this form.
However, as we will see in our experiments in Section \ref{sec:experiment}, these methods do not have global optimality guarantees and may converge to local minima.

\subsection{Quadratic formulation}
To mitigate the shortcomings of local search methods, we can reformulate the cost and constraints of Problem (\ref{equ:original}) as quadratic forms and employ Lagrangian duality to obtain a convex problem which yields a certifiably globally optimal solution.
For the cost function, we can rewrite the original alignment problem with this identity:
\begin{equation}
    (\mathbf{R}^e_w\mathbf{v}^w_{r})^T\mathbf{n}=(\textbf{v}^w_r\otimes\textbf{n} )^T \text{vec}(\mathbf{R}^e_w).
\end{equation}
By reorganizing the variable and coefficient matrix as
\begin{gather}
        \mathbf{r}=\text{vec}(\mathbf{R}^w_e), \quad
            \mathbf{m}=\textbf{v}^w_r\otimes\textbf{n}, 
\end{gather}
we can write the cost with a homogeneous variable $\bar{\mathbf{r}}=\begin{bmatrix}
    \mathbf{r}^T & 1
\end{bmatrix}^T$ as
\begin{equation}
\label{equ:objective_combined}
\begin{aligned}
    &\sum_{i=1}^K \|z_i\|^2_\epsilon = \begin{bmatrix} t \\ \mathbf{\bar{r}} \end{bmatrix}^T 
    \begin{bmatrix} \mathbf{Q}_{t,t} & \mathbf{Q}_{t,\mathbf{\bar{r}}} \\ \mathbf{Q}_{\mathbf{\bar{r}},t} & \mathbf{Q}_{\mathbf{\bar{r}},\mathbf{\bar{r}}} \end{bmatrix}
    \begin{bmatrix} t \\ \mathbf{\bar{r}} \end{bmatrix}, \\
     &\mathbf{Q}_{t,t} = K, \quad \mathbf{Q}_{t, \bar{\mathbf{r}}} = \sum^K_{i=1} \begin{bmatrix} \mathbf{m}_i^T & -\bar{D}_i \end{bmatrix}, \\
    &\mathbf{Q}_{\bar{\mathbf{r}},\bar{\mathbf{r}}} = \sum^K_{i=1} \begin{bmatrix} \mathbf{m}_i\mathbf{m}_i^T & -\bar{D}_i\mathbf{m}_i \\ -\bar{D}_i\mathbf{m}_i^T & \bar{D}_i^2 \end{bmatrix}.
\end{aligned}
\end{equation}




To reduce the computational complexity of our method, we can marginalize the clock drift rate $t$ using Schur's complement: the marginalized cost function is

\begin{equation}
\begin{aligned}
&\sum_{i=1}^K||\bar{z}_i||^2_\epsilon =\bar{\mathbf{r}}^T\bar{\mathbf{Q}}\bar{\mathbf{r}},\\
&\bar{\mathbf{Q}}=\mathbf{Q}_{\mathbf{\bar{r}},\mathbf{\bar{r}}}-\mathbf{Q}_{t,\mathbf{\bar{r}}} \mathbf{Q}_{t,t} ^{-1}\mathbf{Q}_{\mathbf{\bar{r}},t}.
\end{aligned}
\end{equation}
The constraint $\mathbf{R}^e_w\in \mathrm{SO(3)}$ is equivalent to \cite{briales_convex_2017} 
\begin{equation}
\mathrm{SO(3)}\equiv\{\mathbf{R}\in \mathbb{R}^{3\times3}:\mathbf{R}^T\mathbf{R}=\mathbf{I}_3,\det (\mathbf{R})=1\}.
\end{equation}
Here the orthonormal constraint is quadratic but the determinant constraint is cubic. According to \cite{briales_convex_2017}, we can use the quadratic handedness constraints under the right-hand rule to maintain the determinant constraint. The function of these constraints is explained in Section \ref{sec:alignment}. We can therefore define the homogeneous variable $\mathbf{x}=\begin{bmatrix}    \mathbf{r}^T & y, \end{bmatrix}^T$ and derive a primal QCQP problem equivalent to the original Problem (\ref{equ:original}):
\begin{equation}
\label{equ:qcqp}
\begin{aligned}
\min_{\mathbf{x}} \quad & \mathbf{x}^T \bar{\mathbf{Q}} \mathbf{x} \\
\text{s.t.} \quad & \mathbf{R}^T \mathbf{R} = y^2 \mathbf{I}_3, \\
& \mathbf{R} \mathbf{R}^T = y^2 \mathbf{I}_3, \\
& \mathbf{R}^{(i)} \times \mathbf{R}^{(j)} = y \mathbf{R}^{(k)}, \\
& y^2 = 1, \\
& \text{with } i,j,k = \text{cyclic}(1,2,3).
\end{aligned}
\end{equation}

\section{Certifiable alignment}
\label{sec:alignment}
\subsection{Lagrangian duality}

To solve the primal QCQP problem, we employ Lagrangian duality to relax the QCQP to a dual SDP problem, which is concave with a unique optimum:
\begin{equation}
\label{equ:dual}
    \max_{\lambda,\gamma,M,N} \gamma
\end{equation}
\[  
s.t. \ \bar{\mathbf{Q}}+\mathbf{P}(\lambda,M,N,\gamma)\succeq \mathbf{0},
\]
where $\lambda, M,N$ and $\gamma$ are Lagrangian multipliers. For details of $\mathbf{P}(\lambda,M,N,\gamma)$, readers can refer to\cite{briales_convex_2017}. 

According to weak duality theory, the dual problem provides a lower bound for the cost of the original MAP inference in \eqref{equ:original} and its equivalent QCQP in \eqref{equ:qcqp}. Under constraint qualifications such as {Slater's condition}\cite{boyd_convex_2004}, {strong duality} holds, allowing the dual objective to reach the exact global minimum of the primal problem. This relationship is fundamentally governed by the {Karush–Kuhn–Tucker (KKT) conditions}, which serve as the necessary (and also sufficient under Slater's condition) requirements for strong duality\cite{cifuentes_local_2022}:

(i) all constraints in (\ref{equ:qcqp});

(ii) $\mathbf{H} \succeq 0$;

(iii) $\mathbf{H}\mathbf{x} = 0$
where $\mathbf{H}=\bar{\mathbf{Q}}+\mathbf{P}(\lambda,M,N,\gamma)$ is the coefficient matrix of the constraints in (\ref{equ:dual}).

However, in robot perception, reaching the global minimum (\textit{cost-tightness}) does not automatically guarantee the recovery of a unique state (\textit{rank-tightness}). To recover a definitive robot state, rank-tightness is further required (i.e., $\text{corank}(\mathbf{H}^*) = 1$). It has been proven that redundant constraints are contributive to achieving both cost and rank-tightness\cite{dumbgen_globally_2024}. Following the geometric framework in\cite{karzan_exactness_2021}, these constraints expand the dual cone to contract its polar cone so that it could tightly recover the convex hull of the original non-convex problem.

Also, the tightness dual relaxation is affected by some factors, e.g. noises and observability. 
The effect of noise has been proved in\cite{cifuentes_local_2022} and evaluated in many works, including\cite{giamou_certifiably_2019,rosen_se-sync_2019}. Except the disturbance from noises, observability is also an important factor. An frame alignment method is required to be accurate and optimal even with limited measurements. In the following, we focus on the connection between dual relaxation and observability.

\subsection{Dual relaxation tightness and observability} \label{sec:duality_gap}
Although GNSS provides consistent measurement, its signal could degenerate in blocked environments including urban canyons and forests. Therefore, it is necessary to analyze the rank-tightness under poor observability conditions. 
Inspired by some robotic researches, such as\cite{lv_observability-aware_2022,giamou_certifiably_2019}, it is important to analyze the observability to guarantee the existence of a unique solution. Existing works have shown the benefit of redundant constraints on strong duality\cite{dumbgen_globally_2024,briales_convex_2017}. However, this is based on that observability is fulfilled. 
Therefore, we have to guarantee the observability to be the basis of the relaxation tightness. Here we analyze the necessary conditions for observability and discuss about the minimum requirements for satellite visibility and motion of the carrier.

\begin{lemma}[Necessary conditions of observability with redundant constraints] \label{lemma:obs_abundant}
The unknown variables in problem (\ref{equ:dual}) with redundant constraints is observable with the coefficient matrix's degree of freedom (DOF) $\ge4$. To obtain instantaneous alignment, we need at least 2 satellites and velocity along 2 axis.
\end{lemma}

For the observability without redundant constraints, it is more challenging.

\begin{lemma}[Necessary conditions of observability without redundant constraints] \label{lemma:wo_obs_abundant}
The unknown variables in dual problem (\ref{equ:dual}) without redundant constraints is observable with the coefficient matrix's DOF $\ge10$, To obtain instantaneous alignment, we need at least 3 satellites and velocity along 3 axis.
\end{lemma}
For concise writing, the proof of both Lemma \ref{lemma:obs_abundant} and \ref{lemma:wo_obs_abundant} can be found in Appendices.

\section{Experiments}\label{sec:experiment}
We evaluated the proposed certifiable alignment method through simulation and real-world experiments. Simulations enabled ablation studies on motion patterns, redundant constraints, initial estimates, velocities and noise to assess observability and accuracy. Real-world tests in urban environments demonstrated robustness under realistic GNSS conditions. The method was compared with two Doppler-based baselines. The first baseline is a traditional velocity vector registration approach \cite{zhang_vboba_2020} using single point velocity (SPV) estimates \cite{groves_principles_2013} aligned via SVD:

\begin{equation}
\label{equ:voba}
\min_{\mathbf{R}^e_w\in \mathrm{SO(3)}}\sum_{i=1}^N||\mathbf{R}^e_w\mathbf{v}^w_{r,i}-\mathbf{v}^e_{r,i}||^2_\epsilon.
\end{equation}
The second baseline, adapted from \cite{cao_gvins_2022}, directly uses raw Doppler measurements. Originally for $\mathrm{SO(2)}$, we generalized it to $\mathrm{SO(3)}$ to handle full 3D rotations. It estimates initial attitude via Gauss-Newton optimization of Doppler residuals across multiple satellites and epochs.

\begin{table}[]
\centering
\caption{Simulation trajectory Settings}
\label{tab:sim_config}
\begin{tabular}{@{}cc@{}}
\toprule
Parameter         & Value       \\ \midrule
Interval          & 1 $\mathrm{s}$    \\
Duration          & 10 $\mathrm{s}$  \\
Signal Frequency  & 1575.42 $\mathrm{MHz}$ \\
Receiver velocity & 5 $\mathrm{m/s}$       \\
Orbit type        & Walker      \\
Clinical angle    & 55 $\mathrm{\degree}$         \\
Elevation angle   & $\ge$10 $\mathrm{\degree}$        \\ \bottomrule
\end{tabular}
\end{table}

\subsection{Simulation setup}

\begin{figure}
    \centering
    \includegraphics[width=0.7\linewidth]{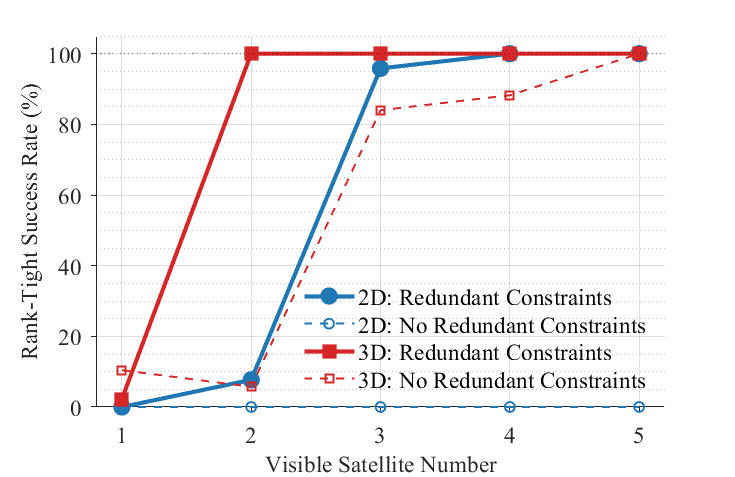}
    \caption{Optimality success rate in 2D/3D motion. The success is confirmed when the ratio of the smallest eigenvalue to the second-smallest eigenvalue of the coefficient matrix $\mathbf{H}$ is smaller than $10^{-6}$. Note that ``false positive'' instances may occur in severely degraded environments (e.g., satellite number $=1$ in the 3D: No Redundant Constraints case), where the rank-tightness condition is numerically satisfied but the resulting estimate remains suboptimal due to insufficient constraints.}
    \label{fig:constraints_line_plot}
\end{figure}

To empirically evaluate the observability lemmas in Section \ref{sec:duality_gap}, we design 2D planar and 3D synthetic trajectories to simulate ground and aerial robotic maneuvers. The parameters in TABLE \ref{tab:sim_config} are selected to reflect standard Global Positioning System (GPS) Medium Earth Orbit (MEO) conditions, using a Walker constellation at 26,560 km altitude. To ensure a realistic signal environment, an elevation mask angle of 10$\degree$ is applied to exclude multipath-prone low-altitude signals. The 10 s observation duration was chosen to provide sufficient local velocity excitation and a diversity of satellite-to-receiver line-of-sight (LOS) vectors, which is critical for verifying the global optimality and observability of the system under different redundant constraints. Specifically, the 2D motion follows a circular path with a 5 m radius, while the 3D motion consists of a constant velocity profile along the local x-axis for 6.67 s followed by a vertical ascent along the z-axis for 3.33 s. Unless otherwise specified, the mean velocity for all trajectories is maintained at 5 m/s.
  
To mitigate the effects of statistical randomness, each test was conducted 200 times within a Monte Carlo framework to ensure the reliability and reproducibility of the results. Using the resulting simulation data, we performed a systematic analysis across key factors influencing observability and alignment performance, including redundant constraints, satellite visibility, motion modes, initial state guesses, and Doppler measurement noise levels. In Sections \ref{sec: redundant_test} and \ref{sec: initial_test}, we utilize noiseless measurements to evaluate the optimality of compared methods and the impact initial receiver positions. In Section \ref{sec: noise_test}, we also assess the robustness of the proposed alignment method and baseline approaches under mild perturbations in the Doppler measurements.

\subsection{Redundant Constraints and Observability} \label{sec: redundant_test}
We utilize noiseless data to evaluate Lemmas \ref{lemma:obs_abundant} and \ref{lemma:wo_obs_abundant}, comparing the rank-tightness success rate across 3D and 2D trajectories. As shown in Fig. \ref{fig:constraints_line_plot}, the proposed method with redundant constraints achieves 100$\%$ success with only two satellites in 3D motion and four satellites in 2D motion, demonstrating significantly enhanced observability. In contrast, without redundant constraints, 2D motion fails to achieve any valid optimality, while 3D motion requires at least five satellites to achieve 100\% success. Notably, for the one visible satellite case in 3D motion without redundant constraints, the plot shows a non-zero success rate. However, since orientation is theoretically unobservable with a single satellite, this is identified as a numerical ``false positive'' where the rank-tight condition is accidentally satisfied without reaching the true global optimum. These results confirm that while 3D dynamics improve reliability, redundant constraints are the critical factor for ensuring genuine global optimality in GNSS-degraded or planar motion environments. Consequently, they are employed by default in our framework to prevent unidentifiable rotation variables \cite{dumbgen_globally_2024}.

\begin{figure}
    \centering
    \includegraphics[width=0.8\linewidth]{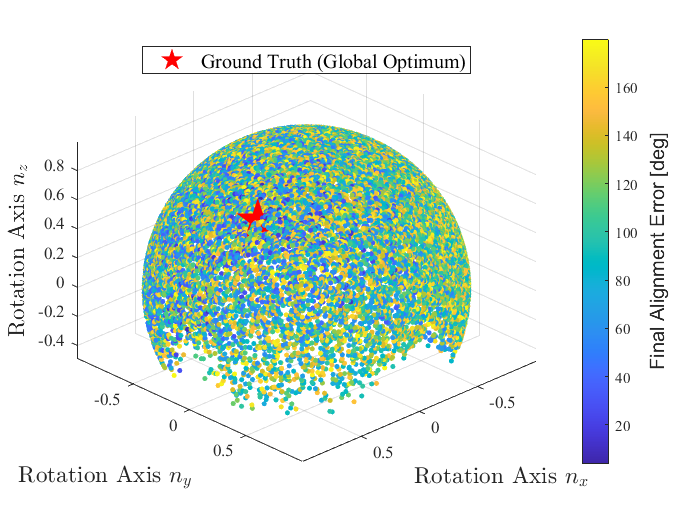}
    \caption{Orientation sensitivity heatmap of the local solver generalized from GVINS, under degenerate satellite visibility (average satellite number $=2$). The axes ($n_x, n_y, n_z$) represent the normalized components of the equivalent rotation axis of the initial orientation error, while the color scale indicates the final alignment error in degrees after convergence. The red star denotes the global optimum (ground truth).
}
    \label{fig:initial_guess}
\end{figure}

 \begin{figure}
        \centering
        \includegraphics[width=1\linewidth]{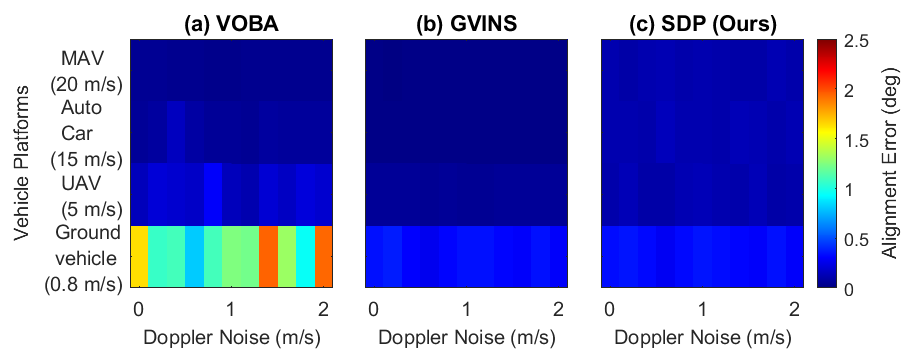}
        \caption{Doppler noise disturbance test with 2D trajectories on four different platforms. Doppler noise level denotes the standard deviation of the Gaussian white noise.}
        \label{fig:3d_noise}
    \end{figure}

\subsection{Effect of initialization} \label{sec: initial_test}
  
In this section, we investigate the effect of the initial state estimate on alignment performance. It is worth noting that both the traditional VOBA method and our SDP-based alignment method do not require an initial guess. Therefore, we focus on the advanced alignment method adapted from GVINS, evaluating it under a degenerate condition with an average of only two visible satellites per epoch and randomly initialized states sampled from the $\mathrm{SO(3)}$ manifold.The results of this evaluation are illustrated in Fig. \ref{fig:initial_guess}, where the red star represents the ground truth, and the heatmap projected over a unit spherical surface depicts the final alignment error in degrees across different initializations. To visualize the rotation space, each point $(n_x, n_y, n_z)$ on the sphere represents the unit direction of the equivalent rotation axis of the initial orientation guess. Two key observations emerge from these results. First, when the geometric distribution of satellites is poor (e.g. with less than four visible satellites), a high-quality initial guess becomes critical for achieving accurate alignment. Second, even initial points located near the ground truth can produce large alignment errors, highlighting the inherent difficulty in identifying feasible initial states under degenerate satellite configurations. These findings underscore the importance of initialization methods that are less sensitive to the initial guess, particularly in scenarios with sparse or poorly distributed satellite visibility.

\subsection{Alignment errors with noise disturbances} \label{sec: noise_test}

In this section, we evaluate the performance of the GVINS alignment method when provided with ground truth as the initial state to establish a performance upper bound. To investigate the robustness of each method under measurement noises, a comparative study was conducted across ten Doppler shift Gaussian white noise levels with the standard deviation from 0 to 2 $m/s$. Furthermore, we analyzed how vehicle dynamics influence alignment by testing four representative platforms (e.g., ground robot, unmanned aerial vehicle (UAV), autonomous car, and micro aerial vehicle (MAV)) at their typical operational velocities. The results for 2D trajectories are illustrated in Fig. \ref{fig:3d_noise}, respectively. The results indicate that under various noise conditions and with exact initialization, the GVINS-based local solver and the proposed SDP method exhibit nearly identical performance, with mean errors below $1^\circ$. Notably, the local solver occasionally achieves slightly higher precision due to its iterative refinement nature, whereas the SDP method's precision is sometimes limited by the numerical tolerances of the interior-point solver (e.g., SeDuMi). A key observation across all methods is that higher vehicle velocity significantly enhances geometric observability, thereby providing greater resistance to increased Doppler measurement noise. However, the traditional VOBA method remains highly sensitive to noises in low-dynamic scenarios. For instance, at a speed of 0.8 m/s, its estimation error frequently exceeds 1°, showing a faster performance degradation compared to optimization-based methods. In contrast, while the GVINS method maintains stable performance given an perfect initialization, the SDP method's optimality comes with a unique trade-off. As illustrated in Fig. \ref{fig:constraints_line_plot}, in degraded GNSS environments where the number of visible satellites is low (e.g., visible satellite number $< 4$), the rank-tightness success rate drops significantly. In these cases, the SDP framework is more ``cautious''. If the global optimality certificate is not obtained, it may withhold a solution rather than risk an incorrect one. While this ensures that any output solution is globally optimal, it also means the success rate under extreme degradation remains lower than that of initialized local methods. Improving the success rate of SDP-based alignment in such challenging conditions remains a promising direction for future research.

\begin{figure}
    \centering
    \includegraphics[width=0.7\linewidth]{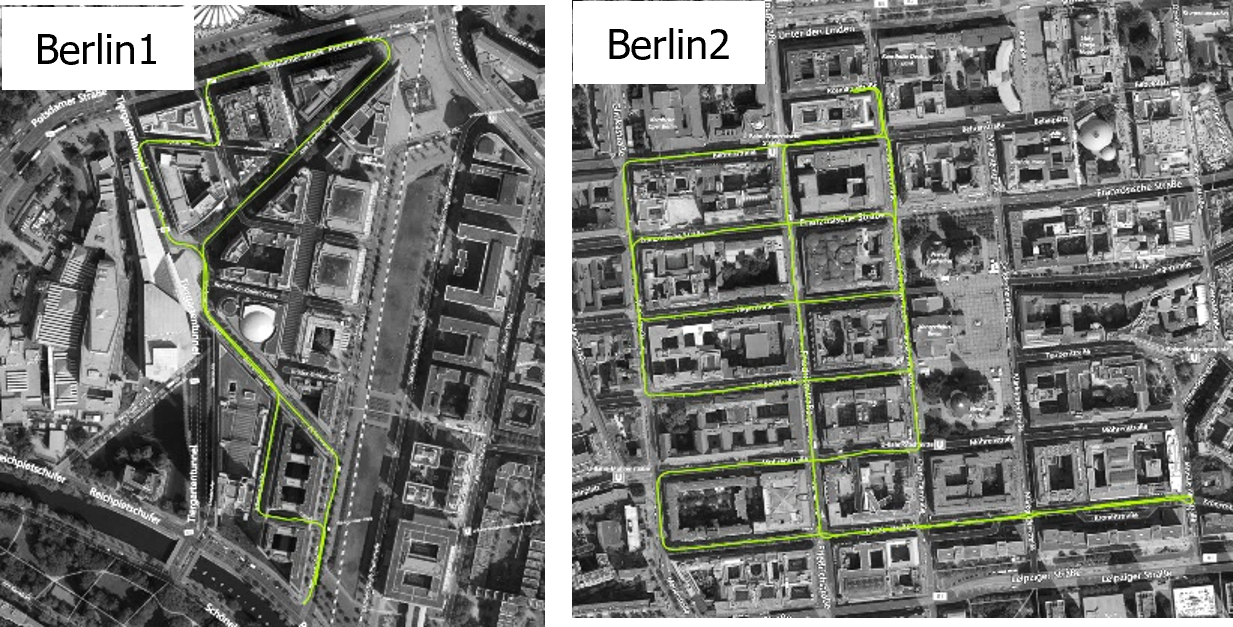}
    \caption{Dataset scenes and vehicle trajectory (green lines) in the real tests}
    \label{fig:real_scene}
\end{figure}

\begin{table}[]
\centering
\caption{hardware settings of the real tests}
\label{tab:real_dataset}
\begin{tabular}{@{}ccc@{}}
\toprule
                & Hardware        & Features                \\ \midrule
Ground-truth    & NovAtel   SPAN  & Yaw   only              \\
GNSS   receiver & U-blox   M8T    & L1                      \\
GNSS   Antenna  & Patch   antenna & GPS,   GLONASS  measurements (5 $\mathrm{Hz}$)    \\
odometer        & CAN             & velocity   and yaw-rate (50 $\mathrm{Hz}$) \\ \bottomrule
\end{tabular}
\end{table}

\subsection{Real test}
As shown in TABLE \ref{tab:real_dataset}, we evaluated the proposed alignment method on the smartLoc dataset\cite{reisdorf_problem_2016}, which provides NovAtel SPAN ground truth and vehicle ego-motion (velocity and yaw rate) from the CAN bus. Low-cost GNSS receivers supply GPS/GLONASS pseudorange and Doppler at L1. As shown in Fig. \ref{fig:real_scene}, two trajectories including Berlin 1 (with average velocity magnitude of 5.84 $m/s$) and Berlin 2 (with average velocity magnitude of 7.08 $m/s$). Since ground truth only includes heading, we focus on alignment error in this dimension, using a 120-second sliding window for statistics. GVINS is initialized with the identity matrix, and mean processing times per epoch are 0.02 s, 0.02 s, and 0.25 s for VOBA, GVINS, and SDP, respectively. Since the proposed methods is primarily designed for one-time system initialization or re-localization rather than continuous tracking, its processing latency of 0.25 s is well within the acceptable range for practical robotics pipelines.

The results, summarized in TABLE \ref{tab:real_result}, include mean absolute errors (MAE), standard deviation of the errors (STD), and maximum errors (MAX). To analyze the effect of sampling frequency, we also downsampled the measurements by a factor of 10 and remarked the tests as Berlin1-10 and Berlin2-10. To demonstrate the influence of satellite visibility, we conducted both \emph{GNSS-healthy} and \emph{GNSS-degradation} tests. In the GNSS-healthy test, we use four satellites per epoch. In the GNSS-degradation test, we restricted the average number of visible satellites to 2 to evaluate the robustness of different alignment methods under limited GNSS conditions. The traditional VOBA method is feasible when sufficient ($\ge4$) satellites are employed but fails when fewer than four satellites are visible. Compared to VOBA, the GVINS-based method remains feasible under degraded GNSS conditions but exhibits strong sensitivity to the quality of the initial estimate, especially for degraded cases. Notably, when the GVINS-based local solver is initialized with the ground truth, the resulting alignment error converges to the same magnitude as our SDP-based method (with a negligible difference of $10^{-4}\degree$). This observation suggests that the performance gap between the two methods is not due to local refinement capabilities, but rather the local solver's inability to identify the global optimum under poor initialization. Moreover, downsampling has not influenced the accuracy of VOBA but sometimes affects GVINS. For example, GVINS becomes trapped in local optima in the Berlin2 test but not in the Berlin2-10 test. In comparison, our SDP alignment method remains robust under low satellite visibility and can provide optimality guarantees after estimation. These results suggest that successful alignment in GNSS-degraded environments, such as under bridges, is primarily governed by the geometric observability and the vehicle's motion-induced change in line-of-sight vectors rather than simply increasing the observation amount or satellite count. Even in the Berlin1-10 and Berlin2-10 tests with sparse 2-satellite constellations, the SDP method consistently identifies the global optimum by accumulating geometric information along the trajectory, proving that a robust global solver is more critical for reliable initialization than high-density temporal measurements. This globally optimal coarse estimate ensures that the initial state resides within the convergence basin of subsequent local filters or sliding-window frameworks, providing a reliable ``cold-start'' capability without requiring any prior information.

\begin{table}[t]
\centering
\caption{Alignment Error in Real Tests}
\label{tab:real_result}
\footnotesize 
\setlength{\tabcolsep}{2.5pt} 
\begin{tabular}{@{}lcccccccccc@{}}
\toprule
 & & \multicolumn{3}{c}{VOBA} & \multicolumn{3}{c}{GVINS} & \multicolumn{3}{c}{SDP (Ours)} \\ \cmidrule(lr){3-5} \cmidrule(lr){6-8} \cmidrule(lr){9-11} 
Yaw Error & Sat. & MAE & STD & MAX & MAE & STD & MAX & MAE & STD & MAX \\ \midrule
Berlin1    & 4 & 5.76 & \textbf{2.27} & 11.98 & 4.72 & 2.59 & 10.84 & \textbf{4.69} & 2.57 & \textbf{9.42} \\
           & 2 & -    & -    & -     & 6.05 & 22.98 & 172.25 & \textbf{2.64} & \textbf{2.06} & \textbf{7.25} \\ \midrule
Berlin1-10 & 4 & 5.76 & \textbf{2.27} & 11.98 & 5.16 & 3.66 & 14.57 & \textbf{5.10} & 3.63 & \textbf{14.33} \\
           & 2 & -    & -    & -     & 9.33 & 25.02 & 193.91 & \textbf{4.53} & \textbf{3.20} & \textbf{18.55} \\ \midrule
Berlin2    & 4 & 16.19 & 5.81 & 23.75 & 1.46 & 4.38 & 106.57 & \textbf{1.29} & \textbf{0.87} & \textbf{3.44} \\
           & 2 & -    & -    & -     & 38.90 & 63.06 & 168.70 & \textbf{1.60} & \textbf{1.14} & \textbf{4.39} \\ \midrule
Berlin2-10 & 4 & 16.19 & 5.81 & 23.75 & \textbf{1.66} & \textbf{1.31} & \textbf{7.79}  & \textbf{1.66} & \textbf{1.31} & \textbf{7.79} \\
           & 2 & -    & -    & -     & 33.90 & 61.09 & 179.17 & \textbf{2.26 }& \textbf{1.77} & \textbf{11.98} \\ \bottomrule
\end{tabular}
\end{table}
\section{Conclusion}
\label{sec:conclusion}
In this work, we introduced a certifiable GNSS/local frame alignment method based on Doppler measurements and convex relaxation. By reformulating the alignment task as a nonconvex QCQP and relaxing it into an SDP problem, we established conditions under which certifiable optimality can be guaranteed through relaxation tightness and observability analysis. Both simulations and real-world experiments, including challenging scenarios with as few as two satellites including Doppler measurements per epoch and 2D motion, demonstrate that our method consistently delivers certifiably optimal solutions where conventional VOBA and GVINS approaches may fail or converge to local optima. To facilitate further research and practical adoption of GNSS-based alignment in robotics, we have open-sourced all of our code and data.

This work shows that, with redundant constraints, an SDP solver is robust to low observability scenarios. Interestingly, these factors are not standalone but affect each other. In future work, we propose to further study the performance of our SDP-based method in degenerated scenarios, e.g. with linear motion or only one observable satellite with Doppler measurements.


{
\appendices
\section*{Proof of the Lemma \ref{lemma:obs_abundant}}

Given a bundle of Doppler measurements and ephemeris, we define $M_{1:K}=[\mathbf{v}_i \otimes \mathbf{n}_i]_{i=1}^K =\mathbf{V}\odot \mathbf{N}$. Since $\text{rank}(\mathbf{M})\leq\text{rank}(\mathbf{V})\cdot\text{rank}(\mathbf{N})$, and the Doppler-based alignment problem involves three unknown rotation Euler angles, we require $\text{rank}(\mathbf{M})\ge3$. Given that $0\le \text{rank}(\mathbf{V})\le 3$ and $0\le \text{rank}(\mathbf{N})\le 3$, the minimum requirement for observability is $\text{rank}(\mathbf{V})\ge 2$ and $\text{rank}(\mathbf{N})\ge 2$. If the rank of either matrix is less than 2, the observability requirement is violated. This completes the proof.

\section*{Proof of the Lemma \ref{lemma:wo_obs_abundant}}
Similar to Lemma \ref{lemma:obs_abundant}, the Doppler-based alignment problem in this context involves ten unknown variables (nine rotation matrix elements and one clock drift rate). Considering the known homogeneous variable, the system requires $\text{rank}(\mathbf{M})\ge9$. Consequently, the minimum integer rank for both $\mathbf{V}$ and $\mathbf{N}$ must be full rank, i.e., $\text{rank}(\mathbf{V})=\text{rank}(\mathbf{N})= 3$. This completes the proof.

}

\bibliography{ral}
\bibliographystyle{IEEEtran}

\vfill

\end{document}